\documentclass[11pt,letterpaper]{article}
\usepackage[top=1in, bottom=1in, left=1in, right=1in]{geometry}

\usepackage[utf8]{inputenc} 
\usepackage[T1]{fontenc}    
\usepackage{lmodern}
\usepackage{hyperref}       
\usepackage{url}            
\usepackage{booktabs}       
\usepackage{amsfonts}       
\usepackage{nicefrac}       
\usepackage{microtype}      
\usepackage{xcolor}         

\usepackage{amsmath, amsthm, amssymb}
\usepackage{bm}
\usepackage{color}
\usepackage{ulem}

\let\emph\textit

\def\R{{\mathbb{R}}}

\def\1{\textbf{1}}

\newcommand{\onevct}{\bm{1}}
\newcommand{\zerovct}{\bm{0}}

\newcommand{\zeromtx}{\mathbf{0}}

\newtheorem{theorem}{Theorem}

\newtheorem{lemma}{Lemma}
\newtheorem{definition}{Definition}

\newtheorem{example}{Example Model}

\newcommand{\Prob}[2][]{\mathbb{P}_{#1}\left\{ {#2} \right\}}
\newcommand{\Expect}[2][]{\mathbb{E}_{#1}\left[ #2 \right]}

\newcommand{\abs}[1]{\left\vert {#1} \right\vert}

\newcommand{\norm}[1]{\left\Vert {#1} \right\Vert}
\newcommand{\normsq}[1]{{\norm{#1}}^2}
\newcommand{\normf}[1]{{\norm{#1}}_\text{F}}

\newcommand{\diag}[1]{\operatorname{diag}{#1}}
\newcommand{\tr}[1]{\operatorname{tr}{#1}}

\newcommand{\cl}[1]{\operatorname{cl}{#1}}
\newcommand{\conv}[1]{\operatorname{conv}{#1}}
\newcommand{\vect}[1]{\operatorname{vec}{#1}}

\newcommand{\st}{\operatorname*{subject\; to}}
\newcommand{\maximize}{\operatorname*{maximize}}
\newcommand{\minimize}{\operatorname*{minimize}}

\newcommand{\vct}[1]{{#1}}

\newcommand{\lmax}{\operatorname{\lambda_{tmax}}}
\newcommand{\lmin}{\operatorname{\lambda_{tmin}}}

\newcommand{\Sm}{\mathcal{S}^{n,m}}
\newcommand{\Smp}{\mathcal{S}_{+}^{n,m}}

\newcommand{\Um}{\mathcal{U}^{n,m}}
\newcommand{\Vm}{\mathcal{V}^{n,m}}

\newcommand{\om}{{\otimes m}}

\newcommand{\inprod}[2][]{\left\langle {#1},{#2} \right\rangle}

\newcommand{\yyast}{y^{\ast \om}}

\newcommand{\Vcone}{Carathéodory symmetric tensor cone\xspace}
\newcommand{\sgmm}{\sigma_2^{n,m}}
\newcommand{\sgmc}{\bar{\sigma}_2^{n,m}}

\newcommand{\yast}{{y^\ast}}

\newcommand{\Vast}{{V^\ast}}

\usepackage{enumitem}
\usepackage{multibib} 
\usepackage{natbib}
\usepackage{graphicx}
\usepackage{subcaption}
\usepackage{algorithm}
\usepackage{algorithmic}
\usepackage{xspace} 
\usepackage[font=small,labelfont=bf]{caption}

\makeatletter
\def\algbackskip{\hskip-\ALG@thistlm}
\makeatother

\title{Exact Partitioning of High-order Models \\ with a Novel Convex Tensor Cone Relaxation}

\author{
  \textbf{Chuyang Ke}\\Department of Computer Science\\Purdue University\\\texttt{cke@purdue.edu}
  \and 
  \textbf{Jean Honorio}\\Department of Computer Science\\Purdue University\\\texttt{jhonorio@purdue.edu}
}

\date{}
\begin{document}
\maketitle

\begin{abstract}
In this paper we propose an algorithm for exact partitioning of high-order models.
We define a general class of $m$-degree Homogeneous Polynomial Models, which subsumes several examples motivated from prior literature.
Exact partitioning can be formulated as a tensor optimization problem.
We relax this high-order combinatorial problem to a convex conic form problem.
To this end, we carefully define the \Vcone, and show its convexity, and the convexity of its dual cone.
This allows us to construct a primal-dual certificate to show that the solution of the convex relaxation is correct (equal to the unobserved true group assignment) and to analyze the statistical upper bound of exact partitioning.
\end{abstract}

\allowdisplaybreaks

\section{Introduction}
Partitioning and clustering algorithms have been favored by researchers from various fields, including machine learning, data mining, molecular biology, and network analysis \citep{xu2015comprehensive, cai2015greedy, nugent2010overview, berkhin2006survey}. Although there is no identical criterion, partitioning algorithms often aim to find a group labeling for a set of entities in a dataset equipped with some \emph{pairwise} metric. In general, the goal is to maximize in-group \emph{affinity}, that is, the entities from the same group are more similar to those from different groups \citep{liu2010robust2, huang2012affinity}.
However in many complex real-world networks, pairwise metrics are not expressive enough to capture all the information. One common assumption is that entities interact in groups instead of pairs. For instance, in a co-authorship network, researchers collaborate in small groups and publish papers \citep{liu2005co}. Another example is the air traffic network \citep{rosvall2014memory}, such that a flight may follow a triangular route A-B-C-A. 
In these scenarios, pairwise metrics are not sufficient to handle high-order relationships between entities. Thus it is important to develop a general high-order partitioning algorithm that can better characterize multi-entity interactions in complex networks.

Recent years witnessed a growing amount of literature on high-order problems, most of them investigating hypergraphs and related applications \citep{papa2007hypergraph, agarwal2005beyond, gibson2000clustering, hagen1992new}. A common approach used in hypergraph-related works, is to transform the hypergraph to a pairwise graph by embedding high-order interactions into pairwise affinities, and then apply traditional graph-based partitioning algorithms \citep{leordeanu2012efficient, zhou2007learning}. 

In this paper we propose a novel high-order model class, namely \emph{$m$-degree Homogeneous Polynomial Models ($m$-HPMs)}. Our $m$-HPM class definition employs the use of homogeneous polynomials to carefully construct an $m$-order tensor, which captures the multi-entity affinities in underlying high-order networks. 
We also provide an \emph{exact partitioning} algorithm with statistical guarantees. 
It is worth mentioning that in the case of second order ($m=2$), the partitioning problem reduces to the Minimum Bisection problem, which is known to be NP-hard \citep{GAREY1976237}.
We relax a high-order combinatorial problem to a convex \emph{conic form problem}, and analyze the Karush–Kuhn–Tucker (KKT) conditions for the optimal solution. Conic form problems are a highly general class of convex optimization problems. For example,  \emph{semidefinite programming} is a special case of conic form programs, when the cone is the set of positive semidefinite matrices.   
We prove that as long as certain statistical conditions are fulfilled, exact partitioning in $m$-HPMs can be achieved.

\textbf{Summary of Our Contributions.}
We provide a series of novel results in this paper:
\begin{itemize}
\item Our definition of $m$-HPMs is a contribution. We are providing the first general model class which characterizes multi-entity interactions in various high-order models. Our definition is highly general, subsumes a wide range of high-order models studied in prior literature, and is amenable to analysis. We show that several high-order problems, including high-order counting models, hypergraph cuts / cliques / volumes / conductance, and motif models, belong to the class of $m$-HPMs.
\item We formulate exact partitioning as a high-order combinatorial optimization problem, and relax it to a convex conic form problem by employing carefully-defined novel tensor primal and dual cones. 
\item We construct a novel primal-dual certificate that leads to the optimal solution of the exact partitioning problem. KKT conditions guarantee our solution to be optimal, as long as the statistical conditions are satisfied. We furthermore characterize the statistical upper bound of exact partitioning by analyzing the tensor eigenvalues associated with the optimal solution.
\end{itemize}
\section{Problem Setting and Notation}
\label{section:prelim}

In this section, we introduce the notations that will be used in the paper. 
For any positive integer $n$, we use $[n]$ to denote the set $\{1, \ldots, n\}$.
For clarity when dealing with a sequence of objects, we use the superscript ${(i)}$ to denote the $i$-th object in the sequence, and subscript $j$ to denote the $j$-th entry. For example, for a sequence of vectors $\{x^{(i)}\}_{i \in [n]}$, $x^{(1)}_2$ represents the second entry of vector $x^{(1)}$. The notation $\otimes$ is used to denote outer product of vectors, for example, $x^{(1)} \otimes \ldots \otimes x^{(m)}$ is a tensor of order $m$, such that 
$
(x^{(1)} \otimes \ldots \otimes x^{(m)})_{i_1, \ldots, i_m} = x^{(1)}_{i_1} \ldots x^{(m)}_{i_m}.
$
We use $\onevct$ to denote the all-one vector.

Let $A$ be an $m$-th order $n$-dimensional real tensor, such that $A_{i_1,\ldots, i_m} \in \R$, 
where $i_j \in [n]$ for every $j \in [m]$. Throughout the paper we require $m$ to be a positive even integer; the motivation will be discussed in Section \ref{section:discussions}.

A tensor is \emph{symmetric} if it is invariant under any permutation of its indices, i.e., 
$
A_{\sigma(i_1),\ldots, \sigma(i_m)} = A_{i_1,\ldots, i_m}
$
for any permutation $\sigma: [m] \to [m]$. We denote the space of all $m$-th order $n$-dimensional symmetric tensors as 
$
\Sm := \{A \mid A_{i_1,\ldots, i_m} \in \R, i_j \in [n], j \in [m], A \text{ is symmetric} \}
$.
Note that $\Sm$ is a vector space, with dimension (i.e., the maximum number of different entries) equal to 
$
\dim{\Sm} = \binom{m+n-1}{m}.
$
We denote the constant 
$
M := \dim{\Sm} + 1= \binom{m+n-1}{m} + 1
$.

We use $\sgmm$ to denote the set of $m$-tuples in the form of $\sigma(i_1,i_1,i_2,i_2,\dots,i_{m/2}, i_{m/2})$, for any permutation $\sigma:[m] \to [m]$ and $i_j \in [n]$. We use $\sgmc$ to denote the complement set $\{(i_1,\dots,i_m) \mid (i_1,\dots,i_m) \notin \sgmm\}$.

For symmetric tensors $A,B \in \Sm$, we define the \emph{inner product}  $\inprod[A]{B}$, the \emph{tensor Frobenius norm} $\normf{A}$, and the \emph{tensor trace} $\tr(A)$ respectively as
\begin{align*}
&\inprod[A]{B} = \sum_{i_1,\ldots,i_m = 1}^{n} A_{i_1 ,\ldots, i_m} B_{i_1 ,\ldots, i_m} ,\qquad \\
&\normf{A} = \sqrt{\inprod[A]{A}}, \qquad 
\tr(A) = \sum_{i = 1}^{n} A_{i,\ldots, i} \,.
\end{align*}

For any vector $u \in \R^n$, we denote the corresponding $m$-th order $\emph{rank-one}$ tensor as $u^{\om}$, where 
$(u^{\om})_{i_1 ,\ldots, i_m} = u_{i_1} \ldots u_{i_m}$,
and we denote the set of all $m$-th order $n$-dimensional rank-one tensors as 
\[
\Um := \{ u^\om \mid u \in \R^n \} \,.
\]

For any tensor $A \in \Sm$, we call $A$ a \emph{positive semidefinite (PSD)} tensor, if for every $B \in \Um$, $\inprod[A]{B} \geq 0$. Similarly, $A$ is called \emph{positive definite}, if for every $B \in \Um$, $\inprod[A]{B} > 0$. We denote the set of all $m$-th order $n$-dimensional positive semidefinite tensors as 
\[
\Smp := \{ A \mid  A \in \Sm, \inprod[A]{B} \geq 0, \forall B\in \Um \} \,.
\]

We also introduce the \emph{\Vcone} $\Vm$, which is defined as 
\[
\Vm := \left\{\sum_{i=1}^M A^{(i)} \mid A^{(i)} \in \Um, i \in [M] \right\} \,.
\]

It might be unclear at this time, but in the next section we prove that $\Smp$ and $\Vm$ are well-defined convex cones.  

We define the \emph{maximum tensor eigenvalue} and the \emph{minimum tensor eigenvalue} of $A$, by its variational characterization, as  
\begin{align*}
&\lmax(A) = \sup_{u \in \R^n, \norm{u}= 1} \inprod[A]{u^\om}, \qquad \\
&\lmin(A) = \inf_{u \in \R^n, \norm{u}= 1} \inprod[A]{u^\om}.
\end{align*}
where $\norm{u}$ is the Euclidean norm of vector $u$.

We now introduce the definition of $m$-degree Homogeneous Polynomial Models, or $m$-HPM for short. 

\begin{definition}[$m$-degree Homogeneous Polynomial Model]
For a high-order random model $\mathcal{M}$, let $n$ be the number of entities (each of them belonging to either one of the two groups), $y^\ast \in \{+1, -1\}^n$ be the unobserved true group assignment, $m$ be the order of the model, $p = (p_0, \ldots, p_m)$ be the coefficient parameter, $\sigma^2$ be the variance, $B$ be the entrywise bound, and $W$ be the random affinity tensor associated with the model. 
We say model $\mathcal{M}$ belongs to the class of $m$-HPM$(n,p,\sigma^2,B)$, if $\mathcal{M}$ satisfies the following properties:
\begin{enumerate}[label=\textnormal{(P\arabic*)}]
    \item Expectation Decomposition: 
    $\displaystyle{\Expect{W} = \sum_{k=0}^m p_k \sum_{\substack{z\in \{0,1\}^m \\ \onevct^\top z = k}} \bigotimes_{i=1}^{m} \left(z_i \onevct + (1-z_i) y^\ast \right)}$; \label{MP1}
    \item Variance Boundedness: 
    $\Expect{\normf{W-\Expect{W}}^2} \leq \sigma^2$; \label{MP2}
    \item Entrywise Boundedness: 
    $\abs{W_{i_1,\ldots, i_m}} \leq B$, for all $i_1, \ldots, i_m \in [n]$. \label{MP3}
\end{enumerate}
The goal is to identify the group membership $y^\ast$ from the observed affinity tensor $W$.
\label{def:modelclass}
\end{definition}

{Our definition of $m$-degree Homogeneous Polynomial Model is highly general.} \ref{MP1} requires the expected affinity tensor can be decomposed into a linear combination of rank-$1$ tensors, and \ref{MP2}, \ref{MP3} only require the variance and absolute value to be bounded above. Informally speaking, \ref{MP1} says one could set the expectation of $W$ for each group arbitrarily by choosing proper $p$'s (see Lemma \ref{lemma:transform}). In Section 5, we show that several high-order examples motivated from prior literature, such as high-order counting models, hypergraph cuts models, minimum bisection models, and motif models, belong to the class of $m$-HPMs.


\section{Tensor Cones and Related Lemmas}
\label{section:cone}

In this section, we provide a series of tensor lemmas that will be used in our analysis.
Proofs of the lemmas can be found in Appendix \ref{section:coneproofs}.
First, we start with some general properties of tensors.

\begin{lemma}[Tensor Inner Product]
For any tensor $X = x^{(1)} \otimes \ldots \otimes x^{(m)}$ and $Y = y^{(1)} \otimes \ldots \otimes y^{(m)}$ in $\Sm$, we have
$
\inprod[X]{Y}
=
\prod_{i=1}^m x^{(i)\top} y^{(i)} .
$
\label{lemma:tensor_innerprod}
\end{lemma}


\begin{lemma}[Tensor Norm Inequality]
For any tensor $A \in \Sm$, $\lmax(A) \leq \norm{A}_F$.
\label{lemma:normineq}
\end{lemma}



\begin{lemma}[Positive Semidefinite Tensor Cone]
$\Smp$ is a convex cone.
\label{lemma:convex_cone}
\end{lemma}

\begin{lemma}[\Vcone]
$\Vm$ is a convex cone.
\label{lemma:Vconvexcone}
\end{lemma}

For any cone $\mathcal{K}$, we use $\mathcal{K}^\ast$ to denote its \emph{dual cone}. We present the following lemmas about duality between the positive semidefinite tensor cone and the \Vcone.

\begin{lemma}[Rank-one Tensors]
$\Um \subset \Smp$, and the dual cone of $\Um$ is $\Smp$, i.e., $(\Um)^\ast = \Smp$. 
\label{lemma:rank1}
\end{lemma}

Lemma \ref{lemma:rank1} allows us to prove the following result about the dual cone of the \Vcone.

\begin{lemma}[Dual of Positive Semidefinite Tensor Cone]
The dual cone of $\Smp$ is $\Vm$, i.e., $(\Smp)^\ast = \Vm$.
\label{lemma:dualconeV}
\end{lemma}

Consequently, Lemma \ref{lemma:dualconeV} gives us the following result about the dual cone of the positive semidefinite tensor cone.

\begin{lemma}[Dual of \Vcone]
The dual cone of $\Vm$ is $\Smp$, i.e., $(\Vm)^\ast = \Smp$.
\label{lemma:doubledual}
\end{lemma}

Before the end of the section, we would like to discuss the relationship between the \Vcone and several related tensor cones. Recall that in the definition of \Vcone, the number of terms in the summation is defined as $M = \dim{\Sm} + 1 = \binom{m+n-1}{m} + 1$.
If we change the number of terms $M$ to a fixed positive integer $r$ in the definition, i.e., $\left\{\sum_{i=1}^r A^{(i)} \mid A^{(i)} \in \Um, i \in [M] \right\}$, we arrive at the definition of the \emph{completely decomposable tensor cone} \citep{qi2017regularly}. In particular, if the number of terms is fixed to be $1$, we come back to the definition of $\Um$, the \emph{rank-one tensor cone}. 
 
A completely decomposable tensor cone is convex, only if the number of terms in the summation is at least $M$, i.e., $r \geq M$. Furthermore for any $r > M$, the completely decomposable tensor cone will be equivalent to the \Vcone. The intuition is that the dimension of the $n$-dimensional, $m$-order symmetric tensor space is $\binom{m+n-1}{m}$, and by Carathéodory's theorem, for any single point in the convex hull of $\Um$, it can be written as the summation of at most $\binom{m+n-1}{m} + 1$ points in $\Um$. For details please check Proof of Lemma \ref{lemma:dualconeV} in the Appendix \ref{section:coneproofs}.

Here we highlight the connection between $M$ in the definition of the \Vcone, and the \emph{symmetric tensor rank} in the tensor (multilinear) algebra literature. For any symmetric tensor $A$ in the complex field $\mathbb{C}$, its (complex) symmetric rank is defined as $\min\{r \mid A = \sum_{i=1}^r u^{(i)\om}, u^{(i)} \in \mathbb{C}^n\}$ \citep{comon2008symmetric}. If one limits the discussion to the real field, the real symmetric rank of $A$ is defined as $\min\{r \mid A = \sum_{i=1}^r \lambda_i A^{(i)}, A^{(i)} \in \Um , \lambda_i \in \R \}$. The major difference between the real symmetric rank and the \Vcone is that, each term in the definition of the real symmetric rank has a possibly negative coefficient $\lambda_i$.
Furthermore, \citet{comon2008symmetric} proves that the complex symmetric rank of any tensor in $\Sm$ is at most $\binom{m+n-1}{m}$, and \citet{ballico2014upper} shows that the real symmetric rank is at most $m$ times the complex symmetric rank, that is, $m\cdot \binom{m+n-1}{m}$. 
On the other hand, every tensor in the \Vcone has a real symmetric rank at most $M = \binom{m+n-1}{m} + 1$.

\section{Convex Relaxation and Analysis}
\label{section:analysis} 
In this section we investigate the conditions for exact partitioning the $m$-degree Homogeneous Polynomial Model into two groups of equal size. We say an algorithm achieves \emph{exact partitioning} if the recovered node labels $y$ is identical to the true labels $y^\ast$. 

Our analysis consists of two parts. First we show the exact partitioning problem for $m$-HPMs can be relaxed to a \emph{conic form problem}, a class of convex optimization problems containing \emph{semidefinite programming} as a specific case. In the second part we use \emph{primal-dual certificates} and statistical concentration inequalities to analyze the sufficient conditions of the problem.  

Our algorithm does not require rounding of the solution. Our proof states that if the statistical conditions are satisfied, our optimization problem will always return the integral ground truth $y^\ast$ as the solution.

The balanced clusters assumption is for clarity of presentation. We can relax the last constraint in \eqref{hopt:mle} in the following subsection to $\sum_i y_i = k$, to allow for different cluster sizes. This does not break our analysis with the novel tensor primal and dual cones.

\subsection{Conic Relaxation}
We first consider a greedy approach to partition a $m$-HPM. Given an observed affinity tensor $W$, we try to find a labeling vector $y$, such that
$
\sum_{i_1 ,\ldots, i_m} W_{i_1,\ldots, i_m} y_{i_1} \ldots y_{i_m}
$
is maximized. Using tensor notations introduced in the previous sections, this can be cast as the following optimization problem
\begin{align}
\maximize_{y} \quad  &\langle W, y^{\otimes m}\rangle \,, \quad \nonumber \\
\st \quad &y \in \{+1, -1\}^n , \quad
\sum_i y_i = 0\,.
\label{hopt:mle}
\end{align}
Problem \eqref{hopt:mle} is nonconvex because of the constraint on $y$. The size of the space of possible $y$'s is exponential in terms of $n$. In fact, in the case of second order ($m=2$), the problem reduces to the Minimum Bisection problem, which is known to be NP-hard \citep{GAREY1976237}. 

To relax the problem we denote $Y = y^\om$. Note that every tensor diagonal element $Y_{i,\ldots, i}$ is always $1$ since $m$ is even. By Lemma \ref{lemma:tensor_innerprod}, $\langle Y, \onevct^{\otimes m} \rangle = (\onevct^\top y)^m = 0$. Thus \eqref{hopt:mle} can be rewritten in the following tensor form
\begin{align}
\maximize_{Y} \quad  &\langle W, Y\rangle  \,,\quad \nonumber\\
\st \quad &Y_{\sgmm} = 1 , \quad 
\langle Y, \onevct^{\otimes m} \rangle = 0 , \quad
Y = y^{\otimes m}\,.
\label{hopt:tensor}
\end{align}

The first constraint above ensures that all entries in tensor $Y$ with even number of repeating indices are set to $1$. For example, in the $m=4$ case, this leads to $Y_{1,1,1,1} = Y_{2,2,2,2} = Y_{1,1,2,2} = Y_{1,2,1,2} = \dots = 1$.
On the other hand, the last constraint in \eqref{hopt:tensor} is still nonconvex. We then substitute it with a tensor cone constraint
\begin{align}
\maximize_{Y} \quad  &\langle W, Y\rangle  \,,\quad \nonumber\\
\st \quad &Y_{\sgmm} = 1, \quad
\langle Y, \onevct^{\otimes m} \rangle = 0, \quad
Y \succeq_{\Vm} 0\,,
\label{hopt:primal}
\end{align}
where $\Vm$ is the \Vcone as defined in the previous section. 

Lemma \ref{lemma:Vconvexcone} tells that $\Vm$ is a convex cone, thus \eqref{hopt:primal} is a convex conic form problem. Furthermore it can be seen that $\Vm$ has a non-empty interior, and there always exists some strictly feasible $Y$'s for the problem. 
Lagrangian of \eqref{hopt:primal} is
$
L(Y,V,\eta,A) 
= -\inprod[W]{Y} + \inprod[V]{Y - \onevct^\om} 
+ \eta \inprod[Y]{\onevct^{\otimes m}}
- \inprod[Y]{A} 
= \langle -W+V+\eta \onevct^{\otimes m} - A, Y\rangle - \inprod[V]{\onevct^\om} \,,
$
where $V\in \Sm, \eta\in\R, A \succeq_{\Smp} 0$ are Lagrangian multipliers, subject to the constraint that $V_{\sgmc} = 0$.
Note that $\Smp$ is the dual cone of $\Vm$ by Lemma \ref{lemma:doubledual}. Taking the derivative of $L$ with respect to $Y$, we obtain 
$
\nabla_Y L = -W+V+\eta \onevct^{\otimes m} - A\,. 
$
Setting the derivative to $0$, we obtain $-W+V+\eta \onevct^{\otimes m} = A$. 
Since $A \succeq_{\Smp} 0$, we obtain that $\inf L = - \inprod[V]{\onevct^\om}$ if $V - W + \eta \onevct^{\otimes m}\succeq_{\Smp} 0$, and unbounded otherwise. This leads to the following dual problem
\begin{align}
\minimize_{V,\eta} \quad  &\inprod[V]{\onevct^\om}  \,,\quad \nonumber\\
\st \quad &V_{\sgmc} = 0 \,,\quad
V - W +\eta \onevct^{\otimes m} \succeq_{\Smp} 0 \,.
\label{hopt:dual}
\end{align}
Lemma \ref{lemma:convex_cone} tells that $\Smp$ is a convex cone, thus \eqref{hopt:dual} is also a convex conic form problem.

{We now examine the optimality condition of the primal problem \eqref{hopt:primal} and the dual problem \eqref{hopt:dual}.}
We first list the Karush–Kuhn–Tucker (KKT) conditions for a primal and dual pair $(Y, V, \eta, A)$ to be optimal.
\begin{align}
V - W +\eta \onevct^{\om} - A &= 0 \,, \tag{Stationarity} \label{kkt:stationarity} \\
Y_{\sgmm} = 1, \quad
\inprod[Y]{\onevct^\om} = 0, \quad
Y &\succeq_{\Vm} 0\,, \tag{Primal Feasibility} \label{kkt:pf} \\
V_{\sgmc} = 0,\quad  A &\succeq_{\Smp} 0\,, \tag{Dual Feasibility} \label{kkt:df} \\
\inprod[A]{Y} &= 0 \,. \tag{Complementary Slackness} \label{kkt:cs}
\end{align}

To guarantee $Y^\ast = y^{\ast\om}$ is an optimal solution to the primal problem \eqref{hopt:primal}, all KKT conditions need to be fulfilled.
First note that $Y^\ast$ fulfills \eqref{kkt:pf} trivially 
because $Y_{\sgmm}^\ast = 1$, $\inprod[Y^\ast]{\onevct^\om} = (\yast^\top \onevct)^m = 0$, and $Y^\ast$ is a rank-one tensor.
Next, combining \eqref{kkt:stationarity} and \eqref{kkt:cs}, we obtain that an optimal solution must fulfill
\begin{equation}
\inprod[V - W +\eta \onevct^{\om}]{Y^\ast} = 0 \,.
\label{opt:inprodzero}
\end{equation}
To fulfill \eqref{opt:inprodzero}, we can construct the dual variables $V^\ast, A^\ast, \eta^\ast$ as follows: 
$V^\ast_{i,\dots,i} = \sum_{i_2,\ldots, i_m} W_{i,i_2,\ldots, i_m} y_i^\ast y_{i_2}^\ast \ldots y_{i_m}^\ast$ for every $i\in [n]$, 
$V^\ast_{i_1,\dots,i_m} = 0$ for all other entries,
$A^\ast = V^\ast - W + \eta^\ast \onevct^{\om}$, and $\eta^\ast \to \infty$.
It remains to prove that our construction $(Y^\ast, V^\ast, \eta^\ast, A^\ast)$ fulfills \eqref{kkt:df} and \eqref{kkt:cs}. This gives us the following optimality condition. Proofs in this section can be found in Appendix \ref{appendix:proof_mainthm}.
\begin{lemma}[Optimality Condition]
The primal problem \eqref{hopt:primal} achieves KKT optimality, if 
\[
\inf_{\norm{u} = 1, u\perp \onevct} \inprod[V^\ast - W]{u^\om} \geq 0 \,.
\]
\label{lemma:opt_optimality}
\end{lemma}

The KKT conditions, once fulfilled, guarantee that $Y^\ast = \yyast$ is an optimal solution to the primal problem. 
However there could exist other sets of primal and dual variables satisfy all KKT conditions above. To illustrate this, we construct a set of example primal and dual variables $(\tilde{Y}, \tilde{V}, \tilde{\eta}, \tilde{A})$ as follows: $\tilde{Y} := \onevct^\om$ is the all-one tensor, $\tilde{V}_{i,\dots,i} := \sum_{i_2,\ldots, i_m} W_{i,i_2,\ldots, i_m} $ for every $i\in [n]$, $\tilde{\eta} \to \infty$, and $\tilde{A} := \tilde{V} - W + \tilde{\eta} \onevct^\om$.
One can verify that $(\tilde{Y}, \tilde{V}, \tilde{\eta}, \tilde{A})$ fulfill all KKT conditions above, and as a result, $\tilde{Y} = \onevct^\om$ is an optimal solution to the primal problem, which is undesirable from the perspective of recovery.

In our analysis, for simplicity we define the combinatorial function
\[
f(m,l,k) := \sum_{s = \max(0, k-l)}^{\min(k, m-l)} (-1)^{s} \binom{l}{k-s} \binom{m-l}{s}\,,
\]
and
\[
F(m, p) := \min
\begin{cases}
  \sum_{l=1}^{m} \binom{m-1}{l-1}  \sum_{k=0}^m p_k f(m,l,k) \\
  \sum_{l=0}^{m-1} \binom{m-1}{l}  \sum_{k=0}^m p_k f(m,l,k)\,.
\end{cases}   
\]
In particular, note that $F(m,p)$ is a function of model order $m$ and the parameter vector $p$, and it characterizes the signal / noise level of the underlying model.
We also denote 
$$
\lambda_{\onevct}(A)
=
\inf_{\norm{u} = 1, u \perp \onevct, u\nshortparallel \yast} \inprod[A]{u^\om} \,,
$$ 
where $u$ cannot be a multiple of $y^\ast$.
To ensure that $Y^\ast = \yyast$ is the \emph{unique} optimal solution to \eqref{hopt:primal} and eliminate all other undesirable solutions, we present the following lemma about uniqueness.

\begin{lemma}[Uniqueness Condition]
The primal problem \eqref{hopt:primal} achieves exact recovery and returns the unique optimal solution $Y^\ast = \yyast$, if 
\begin{equation}
\lambda_{\onevct} ({V^\ast - W }) > 0 \,.
\label{hopt:l2}
\end{equation}
\label{lemma:opt_uniqueness}
\end{lemma}

\subsection{Statistical Conditions of Exact Partitioning}
In this section we analyse the regime in which \eqref{hopt:l2} holds with probability tending to $1$. 
Here we present our main theorem on the statistical conditions of exact partitioning. 
\begin{theorem}
Consider any model $\mathcal{M}$ sampled from class $m$-HPM$(n,p,\sigma^2,B)$ with the assumption of $F(m,p) > 0$. If $\frac{(2^{1-m} F(m,p) - p_0)^2}{B^2} = \Omega\left(\frac{\log n}{n}\right)$,  and $\frac{(2^{1-m} F(m,p) - p_0)^2}{\sigma^2} = \Omega\left(n^{1-m}\right)$, then conic form problem \eqref{hopt:primal} partitions the groups correctly, i.e., the true group assignment $y^\ast$ is the optimal solution of \eqref{hopt:primal}, with probability at least $1 - O(1/n)$.
\label{thm:main}
\end{theorem} 

\begin{proof}[Proof Sketch]
The proof can be broken down into two parts. Starting from our dual construction as in Lemma \ref{lemma:opt_optimality} and \ref{lemma:opt_uniqueness}, the random tensor $V^\ast - W$ can be rewritten as $(V^\ast - \Expect{V^\ast}) - (W - \Expect{W}) + (\Expect{V^\ast} - \Expect{W})$. 
In the first part, we analyze the variational characterization of the expected tensor $\Expect{V^\ast - W }$. We show that $\lambda_{\onevct} (\Expect{V^\ast - W })$ can be bounded below by a quantity.
In the second part, we characterize the spectrum of the deviation tensor $V^\ast - \Expect{V^\ast}$ and $W - \Expect{W}$. 
Since the dual variable $V^\ast$ is constructed to be a diagonal tensor, the minimum tensor eigenvalue of $V^\ast - \Expect{V^\ast}$ is related to the smallest element in its diagonal.
For $W - \Expect{W}$, by Lemma \ref{lemma:normineq}, the maximum tensor eigenvalue is related to its Frobenius norm. 
At the end, our goal is to ensure that $\lambda_{\onevct} ({V^\ast - W })$ is greater than $0$ with high probability. This gives us the statistical conditions in terms of $n, p, \sigma$ and $B$, and since $\lambda_{\onevct} ({V^\ast - W }) > 0$, Lemma \ref{lemma:opt_optimality} and \ref{lemma:opt_uniqueness} guarantee exact recovery through solving the convex primal problem \eqref{hopt:primal}.   
\end{proof}

{Theorem \ref{thm:main} provides the sufficient statistical conditions for the high-order exact partitioning problem.} 
Our proof in Theorem \ref{thm:main} says once the statistical conditions are satisfied, optimization problem \eqref{hopt:primal} will return the integral ground truth $Y^\ast = y^{\ast\otimes m}$ with high probability, where $y^\ast$ is the groundtruth in Definition \ref{def:modelclass}, as well as the optimal solution to problem \eqref{hopt:mle}. This means our tensor cone relaxation from \eqref{hopt:mle} to \eqref{hopt:primal} works effectively.


\section{Discussions}
\label{section:discussions}
It is now the best time to discuss the requirement of $m$ being a positive even integer in the previous section. It is known that there exists no odd-degree nonnegative homogeneous polynomial, and there exists no non-zero odd-order positive semidefinite tensor \citep{yuan2014some}. That is, if $m$ is odd, the cone $\Smp = \{\zeromtx\}$ and therefore its dual cone $\Vm = (\Smp)^\ast = \{\zeromtx\}^* = \Sm$, the space of all symmetric tensors. 
Thus, the requirement of $m$ being a positive even integer is necessary for the convex relaxation.

Our analysis in Section \ref{section:analysis} requires the optimal solution to be unique. A natural question is: why is uniqueness important? The reason is that our models are \emph{generative}. In other words, the true group assignment vector $y^\ast$, selected by nature, generates the observed affinity tensor $W$. From an optimization point of view, it is possible that there exists multiple distinct optimal solutions to problem \eqref{hopt:primal}, however we are only interested in the groundtruth $y^\ast$ which generates the model. 

Our analysis in Section \ref{section:cone} and \ref{section:analysis} focuses on the \Vcone $\Vm$ and the positive semidefinite tensor cone $\Smp$. In the tensor literature another commonly used cone is the \emph{Sum-of-Squares (SoS)} cone. 
It is known that for any $m \geq 4$ and $n \geq 2$, the three cones fulfill $\Vm \subset \text{SoS} \subset \Smp$ \citep{luo2015linear}.

We believe the NP-hardness of checking whether a symmetric tensor is in the \Vcone is an open problem. It is known that many tensor problems are NP-hard, for example, \citet[{Theorem 11.2}]{hillar2013most} points out that deciding whether a symmetric $4$-th order is positive semidefinite is NP-hard. One can see that our \Vcone is a subset of the positive semidefinite tensor cone, but not vice versa, which means that the problem of determining membership of $\Vm$ is not harder than the problem of determining membership of $\Smp$. To the best of our knowledge, it remains unknown if conic form programs with the \Vcone constraint can be solved efficiently. 
Furthermore, even if we relax the constraint $Y \succeq_{\Vm} 0$  in \eqref{hopt:primal} to $Y \succeq_{\Smp} 0$, the argument in \citet{hillar2013most} will not work and NP-hardness remains an open problem. This is because of our constraint $Y_{\sgmm} = 1$ in the primal problem. More details about NP-hardness can be found in Appendix \ref{appendix:nphardness}.

We want to highlight that the utility of our novel relaxation procedure is as a proof technique. From a theoretic point of view, it allows us to apply convex optimization tools to characterize the statistical upper limits of exact partitioning for this class of tensor problems. 
Moreover, convexity introduces new optimization methods, including projected gradient ascent and barrier functions, and thus approximation or randomized algorithms remain possible. This could be a future direction on this problem.
In Appendix \ref{appendix:heuristic} we also provide experimental validation of our theorem by using a projected gradient descent solver.

\begin{figure*}[ht!]
\centering
\includegraphics[width=0.9\textwidth]{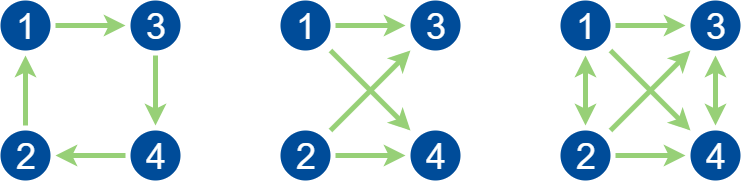}
\caption{Examples of $4$-vertex motifs. \textbf{Left:} The cycle motif. \textbf{Middle:} The ``big-fan'' motif that can be found in neural networks \citep{benson2016higher}, where edges represent information propagation between layers. \textbf{Right:} The food chain motif \citep{li2017inhomogeneous}, where edges represent energy flow between species.}
\label{fig:motif}
\end{figure*}

\section{High-order Example Models}
In this section, we introduce several high-order models motivated from prior literature.
We also show these example models belong to the class of $m$-HPMs by Definition \ref{def:modelclass} in Appendix \ref{appendix:examplemodels}. 
It is worth mentioning that there is no prior theoretical work on exact partitioning in high-order models. We are providing the first results for exact partitioning in models with high-order interactions with provable theoretical guarantees. Our Example Models serve to motivate the generative models, but none of those papers contain any theoretical statistical analysis of exact partitioning.


We first consider the \emph{high-order counting model} motivated from \citep{zhou2007learning}. Suppose there exists a co-authorship network consisting of computer scientists and biologists, and every paper has $m$ authors. On average, authors from the same discipline collaborate more than those from different backgrounds. The task is to identify the two groups of researchers, given the number of publications of each $m$-tuple. Naturally the co-authorship network can be modeled as a high-order counting problem.
Aside from the example above, high-order counting models can be helpful in many complex application problems as pairwise models often lose high-order information, for example, categorical data \citep{gibson2000clustering}, molecular biology \citep{zhang2007random}, and image segmentation \citep{agarwal2005beyond}. 
Next we present a generative model for high-order counting models.

\begin{example}[High-order Counts]
Let $\mathcal{G} = (V,m,\alpha,T)$ be a high-order counting model with vertex set $V$ and order $m$.
$\alpha = (\alpha_0, \ldots, \alpha_{m/2}) \in [0,1]^{m/2+1}$ is the counting parameter vector, and $T \in \mathbb{N}$ is a counting parameter.
Nature generates random counts for $\mathcal{G}$ in the following way. For each $m$-tuple $(v_{i_1},\ldots,v_{i_m}) \subset V$, count the group membership of the vertices. Without loss of generality assume $l$ vertices are from the same group, where $l \in \{0,\ldots,m/2\}$. Nature then samples the corresponding count from binomial distribution: $c(v_{l_1},\ldots,v_{l_m}) \sim \text{\sc Bin} (T, \alpha_l)$. 
We are interested in identifying the group membership of vertices from the observed count information $c(v_{l_1},\ldots,v_{l_m})$. 
\label{eg:count} 
\end{example}


Note that when $T = 1$, the high-order counting model can be interpreted as a random $m$-uniform hypergraph. As a result one can define \emph{hypergraph cuts} as generalization of regular pairwise graph cuts \citep{hein2013total, benson2016higher}. Using a similar approach one can generalize other notions from graph theory, including \emph{clique}, \emph{volume} and \emph{conductance}. Hypergraph cuts have been found useful in tasks dealing with complex networks, for example, video object segmentation \citep{huang2009video}, clustering animals in a zoo dataset using categorical data \citep{zhou2007learning}, among others.
Next we present a generative model for hypergraph cut models.

\begin{example}[Hypergraph Cuts]
Let $G = (V,m=4,\alpha, T=1)$ be a random $m$-uniform hypergraph generated from Model \ref{eg:count}, and let $H$ denote the hyperedge set. For each $4$-tuple $(v_{i_1},v_{i_2},v_{i_3},v_{i_4}) \subset V $, we define its cut size $c(v_{i_1},v_{i_2},v_{i_3},v_{i_4})$, where $c(v_{i_1},v_{i_2},v_{i_3},v_{i_4}) = \sum_{e \in H} \onevct[v_{i_1} \in e \lor v_{i_2} \in e \lor v_{i_3} \in e \lor v_{i_4} \in e] 
\cdot \onevct[(v_{i_1},v_{i_2},v_{i_3},v_{i_4}) \neq e]$.
We are interested in inferring the group membership of vertices. Instead of observing the edge set $E$, we now only observe the cut sizes of every $4$-tuple $(v_{i_1},v_{i_2},v_{i_3},v_{i_4})$. 
\label{eg:cuts}
\end{example}

In the next example we are interested in the \emph{hypergraph minimum bisection} problem. It is well-known that the problem is NP-complete on both pairwise graphs \citep{garey1974some} and hypergraphs \citep{kahng1989fast}.

\begin{example}[Minimum Bisection]
Let $\mathcal{G} = (V,m,q,H)$ be a random $m$-uniform hypergraph with vertex set $V$ and order $m$. $q \in (0,1)$ is a activation parameter.
The hyperedge set $H$ starts empty and nature add hyperedges to $H$ in the following way. Recall that $y^\ast \in \{+1,-1\}^n$ is the group assignment vector. For each $m$-tuple $(v_{i_1},\ldots,v_{i_m}) \subset V$, nature first generates a temporary activation vector $b \in \{+1, -1\}^m$, such that $b_j = y^\ast_{i_j}$ with probability $1-q$, and $b_j = y^\ast_{i_j}$ with probability $q$. If $b_1 = \dots = b_m$, nature adds a hyperedge $(v_{i_1},\ldots,v_{i_m})$ to $H$. Nature then discards the value of $b$, and repeats the process for other $m$-tuples.
We are interested in identifying the group membership of vertices from the hyperedge set $H$.
\label{eg:bisection}
\end{example}


In the next example we investigate \emph{motif models}. Motifs are simple network subgraphs and building blocks of many complex networks \citep{benson2016higher, yaverouglu2014revealing, milo2002network}. Researchers have utilized motifs to explore higher-order patterns and insights in complex systems, such as social networks \citep{juszczyszyn2008local}, air traffic patterns \citep{rosvall2014memory}, and food webs \citep{li2017inhomogeneous, benson2016higher}. Motifs are powerful tools to represent higher-order interaction patterns of multiple entities. Figure \ref{fig:motif} illustrates three distinct motifs of size $4$.

We now present a generative model for motif models. 
\begin{example}[Motif Clustering]
Let $G = (V,\alpha,E,M,H)$ be a directed random graph, such that the vertices $V$ are drawn from two groups $S_1$ and $S_2$. $\alpha = (\alpha_{1,1}, \alpha_{2,2}, \alpha_{1,2}, \alpha_{2,1}) \in [0,1]^{4}$ is a probability parameter vector. The edge set $E$ starts empty and nature add edges to E in the following way. For each pair $(v_1, v_2) \subset V$, 
if $v_1, v_2 \in S_1$, nature adds a directed edge $(v_1,v_2)$ to $E$ with probability $\alpha_{1,1}$; 
if $v_1, v_2 \in S_2$, nature adds a directed edge $(v_1,v_2)$ to $E$ with probability $\alpha_{2,2}$;
if $v_1 \in S_1, v_2 \in S_2$, nature adds a directed edge $(v_1,v_2)$ to $E$ with probability $\alpha_{1,2}$; 
otherwise nature adds a directed edge $(v_1,v_2)$ to $E$ with probability $\alpha_{2,1}$.
$M$ is a $m$-vertex motif of interest, and $H$ is the set of observed motifs. $H$ starts empty. For each $m$-tuple $(i_1, \ldots, i_m) \subset V$, nature adds $(i_1, \ldots, i_m)$ to $H$ if the tuple $(i_1, \ldots, i_m)$ forms the motif $M$ exactly (no extra edges allowed).
We are interested in inferring the group of vertices from the set of observed motifs $H$. 
\label{eg:motif}
\end{example}

\bibliography{0_main.bib}
\bibliographystyle{abbrvnat} 


\newpage
\appendix

\section{Proofs of Tensor Lemmas}
\label{section:coneproofs}

\begin{proof}[Proof of Lemma \ref{lemma:tensor_innerprod}]
By definition of inner products, we obtain
\begin{align*}
\inprod[X]{Y}
&= \sum_{i_1,\ldots,i_m = 1}^{n} X_{i_1 ,\ldots, i_m} Y_{i_1 ,\ldots, i_m} 
= \sum_{i_1,\ldots,i_m = 1}^{n} x^{(1)}_{i_1} \ldots x^{(m)}_{i_m} y^{(1)}_{i_1} \ldots y^{(m)}_{i_m}\\
&= \sum_{i_1,\ldots,i_m = 1}^{n} \prod_{j=1}^{m} x^{(j)}_{i_j} y^{(j)}_{i_j}
= \left(\sum_{i_1 = 1}^n x^{(1)}_{i_1} y^{(1)}_{i_1}\right) \ldots \left(\sum_{i_m = 1}^n x^{(m)}_{i_m} y^{(m)}_{i_m}\right) 
= \prod_{i=1}^m x^{(i)\top} y^{(i)} .
\end{align*}
This completes our proof.
\end{proof}

\begin{proof}[Proof of Lemma \ref{lemma:normineq}]
We use Cauchy-Schwarz inequality in our proof. Note that for any $A \in \Sm$,
\begin{align*}
\lmax(A)^2 
= &\sup_{u \in \R^n, \norm{u}= 1} \inprod[A]{u^\om}^2 
\leq \sup_{u \in \R^n, \norm{u}= 1} \inprod[A]{A} \cdot \inprod[u^\om]{u^\om} \\
= &\norm{A}^2  \sup_{u \in \R^n, \norm{u}= 1} (u^\top u)^m 
= \norm{A}^2 .
\end{align*}
Since norms are nonnegative, we have $\lmax(A) \leq \normf{A}$.
\end{proof}

\begin{proof}[Proof of Lemma \ref{lemma:convex_cone}]
For any $A,B \in \Smp$ and $u \in \mathbb{R}^n$, we have 
\[
\inprod[\theta_1 A + \theta_2 B]{u^{\otimes m}} = \theta_1 \inprod[A]{u^{\otimes m}} + \theta_2 \inprod[B]{u^{\otimes m}} 
\geq 
0,
\]
if $\theta_1, \theta_2 \geq 0$.
\end{proof}

\begin{proof}[Proof of Lemma \ref{lemma:Vconvexcone}]
First note that $\Vm$ is a cone. That is, for every $A \in \Vm$ and $\theta \geq 0$, we have $\theta A \in \Vm$.
Next we show $\Vm$ is convex by considering the convex hull of the set of rank-one tensors $\Um$. Note that for any tensor $A \in \conv(\Um)$, $\dim{\Sm} = M - 1$ is the maximum number of possibly different entries in $A$ due to symmetry.
Applying Carathéodory's theorem to $\Um$ leads to $\Vm = \left\{\sum_{i=1}^M A^{(i)} \mid A^{(i)} \in \Um, i \in [M] \right\} = \conv(\Um)$. This completes our proof.
\end{proof}

\begin{proof}[Proof of Lemma \ref{lemma:rank1}]
To prove the first part, by Lemma \ref{lemma:tensor_innerprod}, for any $u,v \in \mathbb{R}^n$, we have
$
\inprod[u^{\otimes m}]{v^{\otimes m}} 
= 
(u^\top v)^m 
\geq 
0.
$
Thus, by definition of $\Smp$, $u^{\otimes m} \in \Smp$. To prove the second part, by definition of dual cones, we have 
$
(\Um)^\ast 
= \{A \in \Sm \mid \inprod[A]{B} \geq 0, \forall B \in \Um\} 
= \Smp.
$
This completes our proof.
\end{proof}

\begin{proof}[Proof of Lemma \ref{lemma:dualconeV}]
We use $\cl(\cdot)$ to denote the \emph{closure} of a set, and $\conv(\cdot)$ to denote the \emph{convex hull} of a set. 

First we prove that $\Vm$ is a subset of $\Smp$. Note that for any $A \in \Vm$, one can write $A$ as the summation of at most $M$ tensors from $\Um$. In other words, $A = \sum_{i=1}^M A^{(i)}$, where each $A^{(i)} \in \Um$. Since $A^{(i)} \in \Um \subset \Smp$, we have $A = \sum_{i=1}^M A^{(i)} \in \Smp$. Thus $\Vm \subset \Smp$.

Next we prove that $\Vm$ is closed, by showing that if a set contains all limit points, then the set is closed \citep{munkres2014topology}. Without loss of generality assume $A$ is a limit point of $\Vm$, such that $\normf{A} < \infty$. By definition of limit points, $A$ can be approximated by points in $\Vm$. Mathematically this means that one can find an infinite sequence
$\{A_1, A_2, \ldots\}_{j=1}^\infty \subset \Vm$, 
such that $\lim_{j\to \infty} A_j = A$. Since every $A_j$ is in set $\Vm$, by definition there exists a collection of $n$-dimensional vectors $\{x^{(ij)}\}_{i=1}^M \subset \R^n$, such that $A_j = \sum_{i=1}^M (x^{(ij)})^\om$. 
We then consider the tensor Frobenius norm of $A_j$'s. Note that the infinite sequence of tensor Frobenius norm $\{\normf{A_1}, \normf{A_2}, \ldots \}_{j=1}^\infty$ is bounded above with respect to $A$, since $\normf{A}$ is bounded above. Now expanding the $j$-th term in the tensor Frobenius norm sequence, we obtain 
\[
    \sum_{i=1}^M \normf{x^{(ij)}}^{2m} = \normf{A_j}^2 - \sum_{i\neq k} \left( (x^{(ij)})^\top (x^{(kj)}) \right)^m.
\]
Since $m$ is even, every term of the summation on the right-hand side is nonnegative. From the fact that $\{\normf{A_j}\}$ is bounded above and summation is nonnegative, one can tell that the tensor Frobenius norm $\normf{x^{(ij)}}$ is bounded above for every $i$ and $j$. Without loss of generality, we assume there exists $x^{(i)} = \lim_{j\to \infty} x^{(ij)}$ for every $i\in [M]$. It follows that 
\[
A = \lim_{j\to\infty} A_j 
= \lim_{j\to\infty} \sum_{i=1}^M (x^{(ij)})^\om
= \sum_{i=1}^M (x^{(i)})^{\otimes m} .
\]
Then by definition, $A\in \Vm$. Since every limit point of $\Vm$ is contained by itself, topology tells us $\Vm$ is closed, and $\cl(\Vm) = \Vm$.

Note that for any tensor $A \in \conv(\Um)$, $\dim{\Sm}$ is the maximum number of possibly different entries in $A$ due to symmetry. We define a bijective mapping $\vect(\cdot): \Sm \to \R^{\dim{\Sm}}$, which takes a tensor and unfolds it to a vector. Since $M = \dim{\Sm} + 1$, applying Carathéodory's theorem to $\vect(\cdot)$ with basis from $\vect(\Um)$, we have $\Vm = \conv(\Um)$. 
Since $\Vm$ is closed and $\Vm \subset \Smp$, we have $\Vm = \cl(\conv(\Um))$. Also by Lemma \ref{lemma:rank1}, since $\Smp = (\Um)^\ast$, we have $(\Smp)^\ast = (\Um)^{\ast\ast} = \cl(\conv(\Um))$ since for any set $C$, $(C)^{\ast\ast} = \cl(\conv(C))$. Thus we have $(\Smp)^\ast = \Vm$.
\end{proof}

\begin{proof}[Proof of Lemma \ref{lemma:doubledual}]
Note that for any cone $\mathcal{K}$, we have $(\mathcal{K}^\ast)^\ast = \cl(\conv(\mathcal{K}))$. Since $\Smp$ is closed and convex, we have $((\Smp)^\ast)^\ast = \cl(\conv(\Smp)) = \Smp$. By Lemma \ref{lemma:dualconeV} we obtain $(\Vm)^\ast = \Smp$.
\end{proof}

\section{Proof of Main Theorem}
\label{appendix:proof_mainthm}

\begin{proof}[Proof of Lemma \ref{lemma:opt_optimality}]
To prove the lemma, we verify the true primal variable $Y^\ast$ and the contructed dual variables $(V^\ast, A^\ast, \eta^\ast)$ satisfy all KKT conditions. 

Regarding the primal variable, note that $Y^\ast = \yast^\om$, and $\yast$ contains equal number of $+1$'s and $-1$'s. As a result, we have $Y_{\sgmm}^\ast = 1$, and by Lemma \ref{lemma:tensor_innerprod}, $\inprod[Y^\ast]{\onevct^\om} = \inprod[\yast^\om]{\onevct^\om} = (\yast^\top \onevct)^m = 0$. Also note that $Y^\ast$ is a rank-one tensor, which by definition is in the $\Vm$ cone. Thus we have shown the groundtruth $Y^\ast = \yast^\om$ fulfills \eqref{kkt:pf}.

Now we show the constructed dual variables satisfy the other three KKT conditions. First the \eqref{kkt:stationarity} condition is trivially satisfied by our construction
$A^\ast = V^\ast - W + \eta^\ast \onevct^{\om}$.
Next regarding the \eqref{kkt:cs} condition, note that 
\[
\inprod[A^\ast]{Y^\ast}
=
\inprod[V^\ast - W + \eta^\ast \onevct^{\om}]{Y^\ast}    
= 
\inprod[V^\ast]{Y^\ast} - \inprod[W]{Y^\ast} \,.
\]
By construction, $V^\ast$ is a diagonal tensor with $V^\ast_{i,\dots,i} = \sum_{i_2,\ldots, i_m} W_{i,i_2,\ldots, i_m} y_i^\ast y_{i_2}^\ast \ldots y_{i_m}^\ast$ for every $i\in [n]$. This leads to 
\[
\inprod[V^\ast]{Y^\ast} - \inprod[W]{Y^\ast}
=
\sum_i \sum_{i_2,\ldots, i_m} W_{i,i_2,\ldots, i_m} y_i^\ast y_{i_2}^\ast \ldots y_{i_m}^\ast
-
\sum_{i_1,\ldots, i_m} W_{i_1,\ldots, i_m} y_{i_1}^\ast \ldots y_{i_m}^\ast
=
0 \,,
\]
satisfying the \eqref{kkt:cs} condition.

Finally we discuss the \eqref{kkt:df} condition, which requires
$A^\ast = V^\ast - W + \eta^\ast \onevct^{\om}  \succeq_{\Smp} 0$. Recall from Section \ref{section:prelim}, that $A^\ast \in \Smp$, if $\inprod[A^\ast]{u^\om} \geq 0$ for all vector $u$. This is equivalent to requiring
$\inf_{\norm{u} = 1} \inprod[A^\ast]{u^\om} \geq 0$.
We discuss two cases of $u$.
In the first case, for every unit vector $u$ that is not orthogonal to $\onevct$, we have $\inprod[A^\ast]{u^\om} = \inprod[\Vast-W+\eta^\ast \onevct^\om]{u^\om} = \inprod[\Vast-W]{u^\om} + \eta (\onevct^\top u)^m$. 
Regarding the first term, note that each entry of $W$ is bounded by $B$  as in Definition \ref{def:modelclass}, and each entry of $\Vast$ is bounded by $Bn^{m-1}$. By Lemma \ref{lemma:normineq} we obtain
$
\inprod[\Vast-W]{u^\om}
\geq
-\normf{\Vast} - \normf{W}
\geq 
- B n^{m-1/2} - B n^{m/2}
$,
which is a finite quantity bounded below.
Since $m$ is even and $\eta$ is tending to infinity, the summation $\inprod[\Vast-W]{u^\om} + \eta (\onevct^\top u)^m$ is always greater than $0$, or equivalently, $\inf_{\norm{u} = 1, u \not\perp \onevct} \inprod[A^\ast]{u^\om} \geq 0$. 
In the second case, for every unit vector $u \perp \onevct$, we have $\inprod[A^\ast]{u^\om} = \inprod[\Vast-W+\eta^\ast \onevct^\om]{u^\om} = \inprod[\Vast-W]{u^\om}$, which can potentially be negative.
Combining the two cases above, as long as $\inf_{\norm{u} = 1, u\perp \onevct} \inprod[V^\ast - W]{u^\om} \geq 0$, we obtain $\inf_{\norm{u} = 1} \inprod[V^\ast - W + \eta^\ast \onevct^\om]{u^\om} \geq 0$ and equivalently, $A^\ast = V^\ast - W + \eta^\ast \onevct^{\om}  \succeq_{\Smp} 0$. This fulfills the \eqref{kkt:df} condition and completes our proof.
\end{proof}

\begin{proof}[Proof of Lemma \ref{lemma:opt_uniqueness}]
First we want to make sure $Y^\ast = \yyast$ is an optimal solution. By Lemma \ref{lemma:opt_optimality}, it is sufficient to ensure $\inf_{\norm{u} = 1, u\perp \onevct} \inprod[V^\ast - W]{u^\om} \geq 0$. In particular, if $u$ is a multiple of $\yast$, we have $\inprod[V^\ast - W]{\yyast} = 0$, which is exactly the \eqref{kkt:cs} condition.
Now to ensure uniqueness of the solution, we want the equality above to hold only for multiples of $Y^\ast = \yyast$. 
This leads to our condition 
$\inf_{\norm{u} = 1, u \perp \onevct, u\nshortparallel \yast} \inprod[V^\ast - W]{u^\om} 
=
\lambda_{\onevct} ({V^\ast - W})> 0$.
Furthermore, given the constraint $Y_{\sgmm} = 1$ as in \eqref{kkt:pf}, all other multiples of $Y^\ast$ are eliminated from the space of feasible solutions, and the only feasible solution that fulfills all KKT conditions is $Y^\ast$ itself. This completes our proof.
\end{proof}

\begin{proof}[Proof of Theorem \ref{thm:main}]
Lemma \ref{lemma:opt_uniqueness} tells that as long as \eqref{hopt:l2} holds, the conic form problem \eqref{hopt:primal} returns the correct labeling. Our goal is to prove \eqref{hopt:l2} holds with high probability. Note that \eqref{hopt:l2} is a function of random variable $W$. By definition of $\lambda_{\onevct}$, we obtain the following decomposition
\begin{align}
\lambda_{\onevct} (\Vast-W) 
=&\inf_{u \perp \onevct, u\nshortparallel \yast, \norm{u} = 1} 
\inprod[\Vast - W]{u^\om} \nonumber \\
\geq &\inf_{u \perp \onevct, u\nshortparallel \yast, \norm{u} = 1}  
\langle \Vast-\Expect{\Vast}, \vct{u}^{\otimes m}\rangle \label{hsdp:DgivenX_utu} \\
&+\inf_{u \perp \onevct, u\nshortparallel \yast, \norm{u} = 1}  \langle -W+\Expect{W}, \vct{u}^{\otimes m}\rangle \label{hsdp:AgivenX_utu} \\ 
&+\inf_{u \perp \onevct, u\nshortparallel \yast, \norm{u} = 1}  \langle \Expect{\Vast - W}, \vct{u}^{\otimes m}\rangle \label{hsdp:EXgivenX_utu}\,,
\end{align}
and it is sufficient to prove the summation of \eqref{hsdp:DgivenX_utu}, \eqref{hsdp:AgivenX_utu} and \eqref{hsdp:EXgivenX_utu} is greater than $0$.
By definition of tensor eigenvalues, for \eqref{hsdp:DgivenX_utu} we obtain the following lower bound
\begin{align}
\inf_{u \perp \onevct, u\nshortparallel \yast, \norm{u} = 1} 
\langle \Vast-\Expect{\Vast}, \vct{u}^{\otimes m}\rangle 
\geq &\lmin(\Vast-\Expect{\Vast}) \nonumber \\
= &\min_i (V_{i,\dots,i}^\ast - \Expect{V_{i,\dots,i}^\ast}) \,.
\label{hsdp:DgivenX_concentration}
\end{align}
Similarly for \eqref{hsdp:AgivenX_utu}, we have
\begin{align}
\inf_{u \perp \onevct, u\nshortparallel \yast, \norm{u} = 1}
\langle -W+\Expect{W}, \vct{u}^{\otimes m}\rangle 
\geq &\lmin(-W+\Expect{W}) \nonumber\\
= &-\lmax{(W-\Expect{W})} \,.
\label{hsdp:AgivenX_concentration}
\end{align}

Regarding the expectation in \eqref{hsdp:EXgivenX_utu}, we first characterize the expectation of $W$. Consider the definition of \ref{MP1}:
\[
\Expect{W} = \sum_{k=0}^m p_k \sum_{\substack{z\in \{0,1\}^m \\ \onevct^\top z = k}} \bigotimes_{i=1}^{m} \left(z_i \onevct + (1-z_i) y^\ast\right)\,.
\]
Instead of the whole tensor we now consider every single entry $W_{i_1 ,\ldots, i_m}$. By carefully expanding every single entry using combinatorics we obtain
\begin{align*}
\Expect{W_{i_1 ,\ldots, i_m}} \prod_{j=1}^m y_{i_j}^\ast 
=& \sum_{k=0}^m p_k \sum_{\substack{z\in \{0,1\}^m \\ \onevct^\top z = k}} \prod_{j=1}^{m} \left(z_j  + (1-z_j) y_{i_j}^\ast\right) \prod_{j=1}^m y_{i_j}^\ast \\
=& \sum_{k=0}^m p_k \sum_{\substack{z\in \{0,1\}^m \\ \onevct^\top z = k}} \prod_{j=1}^{m} y_{i_j}^{\ast z_j} 1^{1-z_j}  
= \sum_{k=0}^m p_k \sum_{\substack{z\in \{0,1\}^m \\ \onevct^\top z = k}} \prod_{\substack{{j=1} \\ z_j = 1}}^m y_{i_j}^\ast  \\
=& \sum_{k=0}^m p_k 
\sum_{s = \max(0,k-l)}^{\min(k,m-l)}
(-1)^s \binom{l}{k-s}\binom{m-l}{s} \nonumber \\
= &\sum_{k=0}^m p_k f(m,l,k) \,,
\end{align*}
where $l = \sum_{j=1}^m \onevct[y_{i_j}^\ast = 1]$ is the number of positive labels, bounded between $0$ and $m$. The second equality above holds by the fact that we pick every combination of $k$ terms out of $y_{i_1}^\ast$ through $y_{i_m}^\ast$, calculate the product of these $k$ terms (either $+1$ or $-1$), and sum over all possible combinations. Thus by Lemma \ref{lemma:tensor_innerprod}, we obtain
\begin{align}
\sup_{u \perp \onevct, u\nshortparallel \yast, \norm{u} = 1} 
\inprod[\Expect{W}]{u^\om} 
&= \sup_{u \perp \onevct, u\nshortparallel \yast, \norm{u} = 1} 
\inprod[\sum_{k=0}^m p_k \sum_{\substack{b\in \{0,1\}^m \\ \onevct^\top b = k}} \bigotimes_{i=1}^{m} \left(b_i \onevct + (1-b_i) y^\ast\right)]{u^\om} \nonumber\\
&= \sup_{u \perp \onevct, u\nshortparallel \yast, \norm{u} = 1} 
\sum_{k=0}^m p_k \sum_{\substack{b\in \{0,1\}^m \\ \onevct^\top b = k}} \inprod[\bigotimes_{i=1}^{m} \left(b_i \onevct + (1-b_i) y^\ast\right)]{u^\om} \nonumber\\
&= \sup_{u \perp \onevct, u\nshortparallel \yast, \norm{u} = 1} 
\sum_{k=0}^m p_k \sum_{\substack{b\in \{0,1\}^m \\ \onevct^\top b = k}} (u^\top \onevct)^{k} (u^\top \yast)^{m-k} \nonumber\\
&\overset{(a)}{=} p_{0} (u^\top \yast)^{m} \nonumber\\
&\leq p_0 (\normsq{u}\normsq{\yast})^{m/2} \nonumber\\
&= p_0 n^{m/2} \label{hopt:w_bound} \,,
\end{align}
where (a) holds because if $k\neq 0$, there exists at least one term of $u^\top \onevct$, which is $0$ because of orthogonality.

Next we characterize $\Expect{V_{i,\dots,i}^\ast}$. Naturally we consider two cases. If $y_i^\ast$ is positive, the number of positive labels $l$ in $y_i^\ast, y_{i_2}^\ast,\ldots,y_{i_m}^\ast$ is at least $1$, and by definition of $V_{i,\dots,i}^\ast$,
\begin{align*}
\Expect{V_{i,\dots,i}^\ast} 
= &\Expect{\sum_{i_2,\ldots, i_m} W_{i,i_2,\ldots, i_m} y_i^\ast y_{i_2}^\ast ,\ldots, y_{i_m}^\ast} \nonumber \\ 
=& \left(\frac{n}{2}\right)^{m-1} \sum_{l=1}^{m} \binom{m-1}{l-1} 
\sum_{k=0}^m p_k 
f(m,l,k) \,.  
\end{align*}
On the other hand if $y_i^\ast$ is negative, the number of positive labels $l$ in $y_i^\ast, y_{i_2}^\ast,\ldots,y_{i_m}^\ast$ is at most $m-1$. Similarly we obtain
\begin{align*}
\Expect{V_{i,\dots,i}^\ast} 
= &\Expect{\sum_{i_2,\ldots, i_m} W_{i,i_2,\ldots, i_m} y_i^\ast y_{i_2}^\ast \ldots y_{i_m}^\ast} \nonumber\\
=& \left(\frac{n}{2}\right)^{m-1} \sum_{l=0}^{m-1} \binom{m-1}{l} 
\sum_{k=0}^m p_k
f(m,l,k) \,.  
\end{align*}
We want to highlight that, in either case, we have the lower bound
\begin{equation}
\Expect{V_{i,\dots,i}^\ast} \geq \left(\frac{n}{2}\right)^{m-1} F(m, p) \,,
\label{hopt:d_bound}
\end{equation}
for every $i \in [n]$.
Then, combining \eqref{hopt:w_bound} and \eqref{hopt:d_bound}, we obtain the following lower bound for \eqref{hsdp:EXgivenX_utu}
\begin{align}
\inf_{u \perp \onevct, u\nshortparallel \yast, \norm{u} = 1} 
\inprod[\Expect{\Vast - W}]{u^\om} 
\geq &\inf_{u \perp \onevct, u\nshortparallel \yast, \norm{u} = 1} 
\sum_i \Expect{V_{i,\dots,i}^\ast} u_i^m 
- p_0 n^{m/2} \nonumber\\
\geq &\inf_{u \perp \onevct, u\nshortparallel \yast, \norm{u} = 1} 
\left(\frac{n}{2}\right)^{m-1} F(m, p) \sum_i u_i^m - p_0 n^{m/2} \nonumber\\
\geq &\left(\frac{n}{2}\right)^{m-1} F(m, p) n^{1-m/2} - p_0 n^{m/2}  \nonumber\\
\geq & n^{m/2} (2^{1-m} F(m,p) - p_0) \,, \label{hopt:sum_expectation}  
\end{align}
where the second to last inequality holds because $\sum_i u_i^m$ takes the minimum value when $\abs{u_i} = \tfrac{1}{\sqrt{n}}$.

After deriving lower bounds for each of the three terms, we only need to show $\eqref{hsdp:DgivenX_concentration} + \eqref{hsdp:AgivenX_concentration} + \eqref{hopt:sum_expectation}$ is greater than $0$ with high probability. To do so we can simply divide \eqref{hopt:sum_expectation} into two equal parts of size $\frac{1}{2} n^{m/2} (2^{1-m} F(m,p) - p_0)$, and show that 
$
\min_i (V_{i,\dots,i}^\ast - \Expect{V_{i,\dots,i}^\ast}) + \frac{1}{2} n^{m/2} (2^{1-m} F(m,p) - p_0) > 0
$
and 
$
-\lmax{(W-\Expect{W})} + \frac{1}{2} n^{m/2} (2^{1-m} F(m,p) - p_0) > 0 \,.
$

To show \eqref{hsdp:DgivenX_concentration} is bounded below with high probability, we use Hoeffding's inequality in our proof. Note that 
$
\Prob{V_{i,\dots,i}^\ast - \Expect{V_{i,\dots,i}^\ast}\leq -t} 
\leq \Expect{\exp\left(-\frac{2t^2}{(2B)^2 n^{m-1}}\right)} 
= \exp\left(-\frac{t^2}{2B^2 n^{m-1}}\right)\,.
$
By a union bound, we obtain
$
\Prob{\min_i (V_{i,\dots,i}^\ast - \Expect{V_{i,\dots,i}^\ast})\leq -t}
    \leq n \exp\left(-\frac{t^2}{2B^2 n^{m-1}}\right)\,.
$
Setting $t = \frac{1}{2} n^{m/2} (2^{1-m} F(m,p) - p_0)$ leads to
\begin{align}
\Prob{\min_i (V_{i,\dots,i}^\ast - \Expect{V_{i,\dots,i}^\ast}) \leq -\frac{1}{2} n^{m/2} (2^{1-m} F(m,p) - p_0)}  
\leq &n \exp\left(-\frac{n(2^{1-m} F(m,p) - p_0)^2}{8 B^2}\right) \nonumber\\ 
\leq &c_0 n^{-1}\,,
\label{high_prob_1}
\end{align}
and the last inequality holds given that $\frac{(2^{1-m} F(m,p) - p_0)^2}{B^2} 
\geq 16 \frac{\log n}{n} - 8\frac{\log c_0}{n}$.

To show \eqref{hsdp:AgivenX_concentration} is bounded below with high probability, we use the result of Lemma \ref{lemma:normineq}. Note that
$
\Prob{-\lmax(W-\Expect{W})\leq -t} 
\leq  \Prob{-\normf{W-\Expect{W}}^2 \leq -t^2}  
\leq  t^{-2}  \Expect{\normf{W-\Expect{W}}^2}  
\leq  \frac{\sigma^2}{t^2} 
\,.
$
Setting $t = \frac{1}{2} n^{m/2} (2^{1-m} F(m,p) - p_0)$ leads to
\begin{align}
\Prob{-\lmax(W-\Expect{W})\leq -\frac{1}{2} n^{m/2} (2^{1-m} F(m,p) - p_0)} 
\leq &\frac{4 \sigma^2}{n^m (2^{1-m} F(m,p) - p_0)^2} \nonumber\\
\leq &c_1 n^{-1}\,,
\label{high_prob_2}
\end{align}
and the last inequality holds given that 
$\frac{(2^{1-m} F(m,p) - p_0)^2}{\sigma^2} 
\geq \frac{4}{c_1} n^{1-m}$.

Combining the results \eqref{high_prob_1} and \eqref{high_prob_2} above, we can see that $\eqref{hsdp:DgivenX_concentration} + \eqref{hsdp:AgivenX_concentration} + \eqref{hopt:sum_expectation}$ is greater than $0$ with probability at least $1 - (c_0+c_1)n^{-1}$, which means the probability of $\lambda_{\onevct} (\Vast - W) > 0$ is at least $1 - (c_0+c_1)n^{-1}$. 
This completes our proof. 
\end{proof}
 
\section{NP-hardness of Positive Semidefiniteness of a Symmetric Tensor is Open}
\label{appendix:nphardness}
We want to highlight that since program \eqref{hopt:primal} is convex, its correctness can be checked numerically by using a projected gradient descent solver, after relaxing the problem to the positive semidefinite cone. We implement a projected gradient descent solver in Algorithm \ref{alg:new_randomized}, which uses an optimization algorithm to detect whether a given symmetric tensor is positive semidefinite or not \citep{han2013unconstrained} as a subroutine. (Details about our algorithm in the next section.)

We believe the NP-hardness of checking whether a symmetric tensor is in the \Vcone is an open problem. 
Furthermore, we believe the problem is open even if we relax it to the positive semidefinite tensor cone, for any symmetric tensor $S$ subject to the constraint $S_{\sigma_2^{n,m}} = 1$ as in program \eqref{hopt:primal}.
It is claimed in \citet{hillar2013most}, Theorem 11.2, that deciding whether a symmetric $4$-th order is nonnegative definite is NP-hard. However, by looking at their proof right before Theorem 11.2, their definition of tensor symmetry is very different from our definition. In \citet{hillar2013most}, the authors construct a $4$-th order $n$-dimensional tensor $S$ from a symmetric matrix $A = (a_{ij})$, such that $S_{i,j,i,j} = S_{j,i,j,i} = a_{ij}$ for all $i,j\in [n]$, and all other entries are assigned to be $0$. In this way, checking whether tensor $S$ is positive semidefinite (whether $\sum_{ij} a_{ij} x_i^2 x_j^2 \geq 0$ for all $x \in \R^n$) is equivalent to checking whether matrix $A$ is copositive (whether $\sum_{ij} a_{ij} z_i z_j \geq 0$ for all $z \in [0,\infty)^n$), which is known to be a NP-hard problem \citep{burer2012copositive}. As a result, the author claims that checking positive semidefiniteness of tensor fulfilling such symmetry conditions is NP-hard.
However in our case, we have $S_{i,i,j,j} = S_{i,j,j,i} = S_{i,j,i,j} = S_{j,j,i,i} = S_{j,i,i,j} = S_{j,i,j,i} = a_{ij} = 1$, according the constraint $S_{\sigma_2^{n,m}} = 1$ in program \eqref{hopt:primal}. As a result, certifying copositivity is easy because $\sum_{ij} a_{ij} x_i^2 x_j^2 = \sum_{ij} x_i^2 x_j^2 \geq 0$ is trivially true, which means that the NP-hardness reduction above does not work for our constraint $S_{\sigma_2^{n,m}} = 1$.

\section{Projected Gradient Descent Solver}
\label{appendix:heuristic}

In this section, we test the proposed convex optimization formulation \eqref{hopt:primal} by implementing a projected gradient descent solver in MATLAB. 
To do so, we relax the \Vcone to the positive semidefinite tensor cone.
Projected gradient descent is a standard method to solve constrained convex optimization problems.
Our projected gradient descent solver in Algorithm \ref{alg:new_randomized} works as follows: starting from an initial point $Y^{(0)}$, the  algorithm repeats the following assignment until a stopping condition is met:
\[
Y^{(k+1)} \leftarrow P(Y^{(k)} + \zeta W) \,,
\]
where $\zeta$ is the step size of each iteration, and $P(\cdot)$ is a projection operator. 

The projection operator tries to find the ``closest'' point to $Y^{(k)} + \zeta W$, that fulfills all the constraints in the convex problem \eqref{hopt:primal}. 
To fulfill the first two constraints 
$Y_{i,\ldots, i} = 1$
and
$\langle Y, \onevct^{\otimes m} \rangle = 0$,
the algorithm subtracts the mean of all entries in $Y^{(k)}$ from each entry, and sets the diagonal entries to be $1$.
To fulfill the positive semidefinite constraint
$Y \succeq_{\Smp} 0$,
Algorithm \ref{alg:new_randomized} invokes the helper function in Algorithm \ref{alg:new_neg_vector}, using the method in Equation (3.1) in \citet{han2013unconstrained}, which tries to find a unit vector $v \in \R^n$ such that $\inprod[Y^{(k)}]{v^\om}$ is minimized. Algorithm \ref{alg:new_randomized} then subtracts the tensor $\inprod[Y^{(k)}]{v^\om} \cdot v^\om$ from $Y^{(k)}$, making the projected tensor positive semidefinite in the direction of $v$.
We assume the point $P(Y^{(k)} + \zeta W)$ satisfies the constraints, after repeating the projection procedure for a number of iterations.
Theorem 5 in \citet{han2013unconstrained} guarantees that the output vector $v$ is either the zero vector $\zerovct$, or $v$ is a critical point of the function $\inprod[Y^{(k)}]{v^\om}$. If Algorithm \ref{alg:new_neg_vector} returns the zero vector, it is likely that the input $Y^{(k)}$ is already positive semidefinite.

Given an input tensor $A$ (in our case $A = Y$), Algorithm \ref{alg:new_neg_vector} searches for the target vector $x$ (in our case $x = v$) by solving an unconstrained optimization problem as follows 
\[
\min f_1(x) = \frac{1}{2m} \inprod[I]{x^\om}^2 + \frac{1}{m} \inprod[A]{x^\om} \,,
\]
where $I$ is the identity tensor, such that $I_{i,\dots,i} = 1$ and all other entries are $0$ (see \citet[eq. (3.1)]{han2013unconstrained}.
Since both $I$ and $A$ are symmetric, the gradient of the objective function $f_1(x)$ is
\[
\nabla f_1(x) = \inprod[I]{x^\om} Ix^{\otimes m-1} + Ax^{\otimes m-1} \,,
\]
where $Ix^{\otimes m-1}$, $Ax^{\otimes m-1}$ are $n$-dimensional vectors, such that $(Ix^{\otimes m-1})_i = x_i^{m-1}$, and  $(Ax^{\otimes m-1})_i = \sum_{i_2, \dots, i_m} A_{i,i_2,\dots,i_m} x_{i_2} \dots x_{i_m}$.
Algorithm \ref{alg:new_neg_vector} repeats the gradient descent step 
$x \leftarrow x - \gamma \nabla f_1(x)$ and checks the objective value $f_1(x)$. If $f_1(x)$ is less than $0$, the algorithm finds a negative eigenvalue and returns the corresponding eigenvector $x$.

The settings of the experiment are as follows. 
We test Algorithm \ref{alg:new_randomized} using Example Model \ref{eg:count} with $T=1$, i.e., a random $m$-uniform hypergraph model. Note that in this case, tensor $W$ can be interpreted as the adjacency tensor of the $m$-uniform hypergraph.
We pick $n = 20$ to be the number of nodes, and the hypergraph has an order of $m = 4$. 

To run the experiments, we first randomly generate the groundtruth vector $y^{\ast}$ and tensor $Y^\ast$. Next we randomly generate the adjacency tensor $W$ from $Y^\ast$ following the procedure described in Example Model \ref{eg:count}. We then solve for tensor $Y$ using Algorithm \ref{alg:new_randomized}. Our implementation requires the Tensorlab package \citep{tensorlab3.0}. In our experiments we run $100$ outer iterations, $40$ inner iterations, and $20$ descent iterations for each trial. We set the step size $\zeta$ and the gradient descent factor $\gamma$ to be $0.05$. 

\begin{algorithm}
\caption{Projected Gradient Descent Solver}
\label{alg:new_randomized}
\textbf{Input:} Observed affinity tensor $W$, step size $\zeta$  \\
\textbf{Output:} Agreement tensor $\hat{Y}$   
\begin{algorithmic} 
    \STATE $Y \leftarrow \diag(\onevct)$
    \FOR{each outer iteration} 
        \STATE $Y \leftarrow Y + \zeta  W$
        \hfill\COMMENT{gradient descent step}
        \FOR{each inner iteration}
            \STATE $v \leftarrow$ \texttt{find\_neg\_tensor\_evec($Y$)}
            \hfill\COMMENT{invoking Algorithm \ref{alg:new_neg_vector}}
            \STATE $v \leftarrow \frac{1}{\norm{v}} v$
            \STATE $V \leftarrow v^\om$
            \STATE $c_v \leftarrow \inprod[V]{Y}$ \hfill\COMMENT{run projection step if $Y$ is not PSD}
            \IF{$c_v < 0$} 
                \STATE $Y \leftarrow Y - c_v Y$ \hfill\COMMENT{$Y \succeq_{\Smp} 0$}
                \STATE $Y \leftarrow Y - (\frac{1}{n^m} \sum_{i_1,\dots,i_m} Y_{i_1,\dots,i_m}) \onevct^\om$ \hfill\COMMENT{$\langle Y, \onevct^{\otimes m} \rangle = 0$}
                \STATE $Y_{i,\dots,i} \leftarrow 1$ \hfill\COMMENT{$Y_{i,\ldots, i} = 1$}
            \ENDIF
        \ENDFOR
        \STATE $Y \leftarrow Y - (\frac{1}{n^m} \sum_{i_1,\dots,i_m} Y_{i_1,\dots,i_m}) \onevct^\om$
        \STATE $Y_{i,\dots,i} \leftarrow 1$
    \ENDFOR
\end{algorithmic}
\end{algorithm}

\begin{algorithm}
\caption{Negative Tensor Eigenvalue (see Equation (3.1) and Theorem 5 in \citet{han2013unconstrained})}
\label{alg:new_neg_vector}
\textbf{Input:} Target tensor $A$, step size $\gamma$  \\
\textbf{Output:} Target vector $x$   
\begin{algorithmic} 
    \STATE Initialize $x$
    \STATE $f_1 \leftarrow \infty$
    \FOR{each iteration} 
        \STATE $f_1' \leftarrow f_1$ \hfill\COMMENT{record previous $f_1$ value}
        \STATE $f_1 \leftarrow \frac{1}{2m} \inprod[I]{x^\om}^2 + \frac{1}{m} \inprod[A]{x^\om}$ \hfill\COMMENT{check current objective value}
        \IF{$f_1 < 0$ or $\inprod[A]{x^\om} < 0$}
            \STATE Break \hfill\COMMENT{CAN certify negative definiteness}
        \ENDIF
        \IF{$f_1 > f_1'$}
            \STATE Break \hfill\COMMENT{$f_1$ is starting to grow; stop}
        \ENDIF
        \STATE $\nabla f_1 \leftarrow \inprod[I]{x^\om} Ix^{\otimes m-1} + Ax^{\otimes m-1}$
        \STATE $x \leftarrow x - \gamma \nabla f_1$ \hfill\COMMENT{gradient descent step}
    \ENDFOR
    \IF{$f_1 \geq 0$ and $\inprod[A]{x^\om} \geq 0$}
        \STATE{$x \leftarrow \zerovct$}
    \ENDIF
\end{algorithmic}
\end{algorithm}

Once Algorithm \ref{alg:new_randomized} returns the tensor $Y$, we compare it with the groundtruth $Y^\ast$. To do so, we check the value $h(Y) := \inprod[Y]{Y^\ast} / \inprod[Y^\ast]{Y^\ast}$. 
Note that the balanced constraint $\inprod[Y]{\onevct^{\otimes m}} = 0$ from \eqref{hopt:primal} is enforced by our solver. Thus if $Y$ is fully noisy (i.e., every entry is $0$ on average), the value $h(Y)$ will be $0$. 

We test Algorithm \ref{alg:new_randomized} under two different settings of parameters and compare their results. In the first experiment, we set $\alpha = (0.9, 0.1, 0, 0.1, 0.9)$ as the counting parameter vector. In the second experiment, we set $\alpha = (0.6, 0.4, 0, 0.4, 0.6)$. Intuitively, the hypergraph in the second experiment is much noisier than the one in the first experiment. 
We run both experiments for $10$ trials. 
In the first experiment, the average value of $h(Y)$ is $0.298$.
In the second experiment, the average value of $h(Y)$ is $-0.249$. Note that the threshold of $h(Y)$ is $0$ in a fully noisy case. Since the value is significantly greater than the fully noisy threshold in the first case but not in the second case, Algorithm \ref{alg:new_randomized} captures the underlying structure if the signal is strong enough. This matches our findings in Theorem \ref{thm:main}.


\section{Example Models Reparametrized to the Class of \texorpdfstring{$m$}{m}-HPMs}
\label{appendix:examplemodels}
We first provide the following lemma, which characterize the connection between the expectation weights $\alpha$ and the coefficient parameter $p$. 
\begin{lemma}
Let $\Expect{W_{i_1,\dots,i_m}} = \alpha_l$, if $l$ out of $y_1, \dots, y_m$ are equal to $+1$. We denote $\alpha = (\alpha_0, \dots, \alpha_m)$ as the vector of expectation weights of $W$.
Let $L \in \R^{(m+1)\times (m+1)}$, where 
$
L_{ij} = 
\sum_{s = \max(0, j-i)}^{\min(j-1, m-i+1)} 
(-1)^{i+s-1} \binom{i-1}{j-s-1} \binom{m-i+1}{s}
$. 
Then for any $\alpha$, we have $L p = \alpha$ and $L^{-1} \alpha = p$. We say $L$ is the linear transformation between the expectation weights $\alpha$ and the coefficient parameter $p$.
\label{lemma:transform}
\end{lemma}
\begin{proof}
This can be shown by looking at each entry in $\Expect{W}$. Note that
\begin{align*}
\Expect{W_{i_1,\dots,i_m}}
&= \sum_{k=0}^m p_k \sum_{\substack{z\in \{0,1\}^m \\ \onevct^\top z = k}} \prod_{j=1}^{m} \left(z_{j}  + (1-z_{j}) y^\ast_{i_j} \right) \\
&= \sum_{k=0}^m p_k \sum_{\substack{z\in \{0,1\}^m \\ \onevct^\top z = k}} \prod_{j=1}^{m} y_{i_j}^{\ast (1-z_j)} \\
&= \sum_{k=0}^m p_k \sum_{s = \max(0, k-l)}^{\min(k, m-l)} (-1)^{l+s} \binom{l}{k-s} \binom{m-l}{s} \,,
\end{align*}
where $l = \sum_{j=1}^m \onevct[y_{i_j}^\ast = 1]$ is the number of positive labels. Note that both $k$ and $l$ are bounded between $0$ to $m$, where as $i$ and $j$ in the statement are between $1$ and $m+1$. Reparameterizing $k = j-1$ and $l = i-1$, we obtain the expression of $L_{ij}$.
\end{proof}

As a side note, in terms of the expectation weights vector $\alpha$ , the closed form of $F(m, L^{-1}\alpha)$, as defined in the proof of Theorem \ref{thm:main}, is $\min(\inprod[\Xi_m]{(\alpha_0,\dots,\alpha_{m-1}} , \inprod[\Xi_m]{(\alpha_1,\dots,\alpha_{m}})$, where $\Xi_m \in \R^n$ is the $m$-th row of Pascal's triangle.

We now proceed to show the example models in Section 5 belong to the class of $m$-HPMs.

\noindent\textbf{Example Model \ref{eg:count}.}
We first show Model \ref{eg:count} belongs to the class of $m$-HPMs. For simplicity we only consider the case of $m = 4$. Let $n = \abs{V}$, $W$ be the affinity tensor, where $W_{i_1 i_2 i_3 i_4} = c(v_{i_1},v_{i_2},v_{i_3},v_{i_4})$. Without loss of generality we label $y_i = 1$ if vertex $i$ is from the first group, and $-1$ otherwise. 

Model \ref{eg:count} needs to fulfill \ref{MP1}, \ref{MP2} and \ref{MP3}. The latter two are trivial because Model \ref{eg:count} is bounded, and it remains to prove \ref{MP1}.
From the assumption of binomial distribution, $\Expect{W_{i_1 i_2 i_3 i_4}} = T \alpha_l$ if $l$ labels from $y_{i_1}^\ast, y_{i_2}^\ast, y_{i_3}^\ast, y_{i_4}^\ast$ are the same (w.l.o.g. pick the smaller group). To fulfill \ref{MP1}, one needs to find a vector $p = (p_0, \ldots, p_4)$ satisfying both 
$${\Expect{W} = \sum_{k=0}^m p_k \sum_{\substack{z\in \{0,1\}^m \\ \onevct^\top z = k}} \bigotimes_{i=1}^{m} \left(z_i \onevct + (1-z_i) y^\ast \right)}$$
and
$$\Expect{W_{i_1 i_2 i_3 i_4}} = T\alpha_l,$$
for every $i_1, i_2, i_3, i_4 \in [n]$.
By Lemma \ref{lemma:transform} we have the following linear system
\[
Lp = 
\begin{pmatrix}
1 & -4 & 6 & -4 & 1\\
-1 & 2 & 0 & -2 & 1 \\
1 & 0 & -2 & 0 & 1 \\
-1 & -2 & 0 & 2 & 1 \\
1 & 4 & 6 & 4 & 1 
\end{pmatrix}
\begin{pmatrix}
    p_0\\p_1\\p_2\\p_3\\p_4
\end{pmatrix}
=
T
\begin{pmatrix}
    \alpha_0 \\ \alpha_1 \\ \alpha_2 \\ \alpha_1 \\ \alpha_0
\end{pmatrix}.
\]
By solving the linear equation system above we obtain
\[
\begin{pmatrix}
    p_0\\p_1\\p_2\\p_3\\p_4
\end{pmatrix}
=
T
\begin{pmatrix}
    \alpha_0/8 - \alpha_1/2 + 3\alpha_2 /8 \\
    0 \\
    \alpha_0/8 - \alpha_2 /8 \\
    0 \\
    \alpha_0/8 + \alpha_1/2 + 3\alpha_2 /8
\end{pmatrix}.
\]
Thus, by setting $p$ as above, \ref{MP1} is fulfilled and we have shown Model \ref{eg:count} belongs to the class of $4$-HPMs.


\noindent\textbf{Example Model \ref{eg:cuts}.}
The procedure is similar to the previous model. Model \ref{eg:cuts} needs to fulfill \ref{MP1}, \ref{MP2} and \ref{MP3}. Model \ref{eg:cuts} is bounded, thus one only needs to prove \ref{MP1} by finding the expectation of each single entry in $W$ and solving for $p$.
We use $\alpha'_l$ to denote the expected hypergraph cut size of $(v_{i_1},v_{i_2},v_{i_3},v_{i_4})$ if $l$ labels in $y_{i_1}^\ast, y_{i_2}^\ast, y_{i_3}^\ast, y_{i_4}^\ast$ are the same. Then by Lemma \ref{lemma:transform}, solving $p$ in the linear system 
\[
Lp = 
\begin{pmatrix}
    1 & -4 & 6 & -4 & 1\\
    -1 & 2 & 0 & -2 & 1 \\
    1 & 0 & -2 & 0 & 1 \\
    -1 & -2 & 0 & 2 & 1 \\
    1 & 4 & 6 & 4 & 1 
\end{pmatrix}
\begin{pmatrix}
    p_0\\p_1\\p_2\\p_3\\p_4
\end{pmatrix}
=
\begin{pmatrix}
    \alpha'_0 \\ \alpha'_1 \\ \alpha'_2 \\ \alpha'_1 \\ \alpha'_0
\end{pmatrix},
\]
\ref{MP1} is fulfilled and we have shown Model \ref{eg:cuts} belongs to the class of $4$-HPMs.


\noindent\textbf{Example Model \ref{eg:bisection}}
We now show Model \ref{eg:bisection} belongs to the class of $m$-HPMs. For simplicity we only consider the case of $m = 4$. Let $n = \abs{V}$, $W$ be the affinity tensor, where $W_{i_1 i_2 i_3 i_4} = 1$ if $(v_{i_1},v_{i_2},v_{i_3},v_{i_4}) \in H$, and $0$ otherwise. 

Model \ref{eg:bisection} needs to fulfill \ref{MP1}, \ref{MP2} and \ref{MP3}. The latter two are trivial because Model \ref{eg:bisection} is bounded, and it remains to prove \ref{MP1}.
Note that the expectation $\Expect{W_{i_1 i_2 i_3 i_4}}$ depends on the group assignments $y_{i_1}^\ast,y_{i_2}^\ast,y_{i_3}^\ast$ and $y_{i_4}^\ast$. From the model definition one can find that
\begin{itemize}
    \item $\Expect{W_{i_1 i_2 i_3 i_4}} = q^4 + (1-q)^4$, if all four vertices are from the same group;
    \item $\Expect{W_{i_1 i_2 i_3 i_4}} = q(1-q)^3 + q^3(1-q)$, if three vertices are from one group, and one is from the other group;
    \item $\Expect{W_{i_1 i_2 i_3 i_4}} = 2q^2 (1-q)^2$, if two vertices are from one group and two are from the other group.
\end{itemize}
Thus by Lemma \ref{lemma:transform} we can transform the conditions into the following linear system
\[
Lp = 
\begin{pmatrix}
    1 & -4 & 6 & -4 & 1\\
    -1 & 2 & 0 & -2 & 1 \\
    1 & 0 & -2 & 0 & 1 \\
    -1 & -2 & 0 & 2 & 1 \\
    1 & 4 & 6 & 4 & 1 
\end{pmatrix}
\begin{pmatrix}
    p_0\\p_1\\p_2\\p_3\\p_4
\end{pmatrix}
=
\begin{pmatrix}
    q^4 + (1-q)^4 \\ q(1-q)^3 + q^3(1-q) \\ 2q^2 (1-q)^2 \\ q(1-q)^3 + q^3(1-q) \\ q^4 + (1-q)^4
\end{pmatrix}.
\]
Thus, by solving for $p$ in the linear system above, \ref{MP1} is fulfilled and Model \ref{eg:bisection} belongs to the class of $4$-HPMs.


\noindent\textbf{Example Model \ref{eg:motif}.}
In Figure \ref{fig:motif} we show three example motifs of size $4$. For concreteness let us consider the last motif. This motif has been used to model food chains in the Florida Bay food web \citep{li2017inhomogeneous}. In this motif, nodes are considered as species, and the directed edges represent carbon flow, i.e., a directed edge $i \to j$ can be interpreted as species $i$ consumes species $j$. Therefore the motif can capture interaction and energy flow between multiple species in the food web.

We now show Model \ref{eg:motif} with the last motif (food chain) from Figure \ref{fig:motif} belongs to the class of $4$-HPMs. Let $n = \abs{V}$, $W$ be the affinity tensor, where $W_{i_1 i_2 i_3 i_4} = 1$ if $(v_{i_1},v_{i_2},v_{i_3},v_{i_4}) \in H$, and $0$ otherwise. Without loss of generality we label $y_i^\ast = 1$ if vertex $i$ is from group $S_1$, and $-1$ otherwise. 

Model \ref{eg:motif} needs to fulfill \ref{MP1}, \ref{MP2} and \ref{MP3}. Again Model \ref{eg:motif} is bounded, and one only needs to prove \ref{MP1} by finding the expectation of each single entry in $W$ and solving for $p$. 
Note that the expectation $\Expect{W_{i_1 i_2 i_3 i_4}}$ depends on the group assignments $y_{i_1}^\ast,y_{i_2}^\ast,y_{i_3}^\ast$ and $y_{i_4}^\ast$. By careful analysis of combinations, one can find that
\begin{itemize}
    \item $\Expect{W_{i_1 i_2 i_3 i_4}} = \beta_0 := 6\alpha_{1,1}^8 (1-\alpha_{1,1})^4$, if all four labels are positive;
    \item $\Expect{W_{i_1 i_2 i_3 i_4}} = \beta_1 := 3\alpha_{1,1}^4 \alpha_{1,2}^3 \alpha_{2,1} (1-\alpha_{1,1})^2 (1-\alpha_{2,1})^2 + 3\alpha_{1,1}^4 \alpha_{1,2} \alpha_{2,1}^3 (1-\alpha_{1,1})^2 (1-\alpha_{1,2})^2 $, if three labels are positive;
    \item $\Expect{W_{i_1 i_2 i_3 i_4}} = \beta_2 := \alpha_{1,1}^2 \alpha_{1,2}^4 \alpha_{2,2}^2 (1-\alpha_{2,1})^4 + 4 \alpha_{1,1} \alpha_{1,2}^2 \alpha_{2,1}^2 \alpha_{2,2} (1-\alpha_{1,1})(1-\alpha_{1,2})(1-\alpha_{2,1})(1-\alpha_{2,2}) + \alpha_{1,1}^2 \alpha_{2,1}^4 \alpha_{2,2}^2 (1-\alpha_{1,2})^4$, if two labels are positive;
    \item $\Expect{W_{i_1 i_2 i_3 i_4}} = \beta_3 := 3 \alpha_{1,2}^3 \alpha_{2,1} \alpha_{2,2}^4 (1-\alpha_{2,1})^2 (1-\alpha_{2,2})^2 + 3 \alpha_{1,2} \alpha_{2,1}^3 \alpha_{2,2}^4 (1-\alpha_{1,2})^2 (1-\alpha_{2,2})^2$, if one label is positive;
    \item $\Expect{W_{i_1 i_2 i_3 i_4}} = \beta_4 := 6\alpha_{2,2}^8 (1-\alpha_{2,2})^4$, if all four labels are negative.
\end{itemize}
Thus by Lemma \ref{lemma:transform} we can transform the conditions into the following linear system
\begin{align*}
Lp = 
\begin{pmatrix}
    1 & -4 & 6 & -4 & 1\\
    -1 & 2 & 0 & -2 & 1 \\
    1 & 0 & -2 & 0 & 1 \\
    -1 & -2 & 0 & 2 & 1 \\
    1 & 4 & 6 & 4 & 1 
\end{pmatrix}
\begin{pmatrix}
    p_0\\p_1\\p_2\\p_3\\p_4
\end{pmatrix} 
= 
\begin{pmatrix}
    \beta_0 \\
    \beta_1 \\
    \beta_2 \\
    \beta_3 \\
    \beta_4
\end{pmatrix}.
\end{align*}
Thus, by solving for $p$ in the linear system above, \ref{MP1} is fulfilled and Model \ref{eg:motif} belongs to the class of $4$-HPMs.

\end{document}